\documentclass[11pt]{amsart}
\usepackage{graphicx}
\usepackage{textcomp}
\usepackage{xcolor}
\usepackage{cite}
\usepackage[foot]{amsaddr}

\newtheorem{thm}{Theorem}
\newtheorem{prop}{Proposition}

\newtheorem{cor}{Corollary}

\newtheorem*{thmrade}{Rademacher's theorem}

\theoremstyle{definition}
\newtheorem{definition}{Definition}
\newtheorem{example}{Example}
\newtheorem{counterex}{Counterexample}

\theoremstyle{remark}

\newcommand\skipthis[1]{}

\title{Symmetry-Preserving Paths in Integrated Gradients}

\author[1]{Miguel Lerma}
\address[1]{Northwestern University, Evanston, USA}
\email[1]{mlerma@math.northwestern.edu}
\author[2]{Mirtha Lucas}
\address[2]{DePaul University, Chicago, USA}
\email[2]{mlucas3@depaul.edu}

\date{\today}

\begin{document}

\begin{abstract}
  We provide rigorous proofs that the Integrated Gradients (IG)
  attribution method for deep networks satisfies completeness and
  symmetry-preserving properties.  We also study the uniqueness of IG
  as a path method preserving symmetry.
\end{abstract}

\maketitle

\section{Introduction}

\subsection{Integrated Gradients}

The Integrated Gradients (IG) attribution method for deep networks is
introduced in \cite{sundararajan2017ig}.  A neural network can be
interpreted as a function $F : \mathbb{R}^n \to \mathbb{R}$ that maps
its inputs $\mathbf{x} = (x_1,\dots,x_n)$ to an output
$F(\mathbf{x})$.  As an example, assume that the input represents an
image, and the $x_i$ are the intensities of its pixels.  Assume we
want to determine whether the image contains some element, say a dog,
or a cat.  There are deep networks that have been training with
millions of images and can provide an answer by assigning a score
(given by the value of $F(x)$) to the presence or absence of the
chosen element (see e.g.  \cite{simonyan2015vgg}).  If the element is
present in the image then the score is high, if not it is low.  Each
possible element has an associated score function, say
$F_{\text{dog}}(\mathbf{x})$, $F_{\text{cat}}(\mathbf{x})$, etc.  A
problem that IG is designed to solve is to determine the contribution
of each of its inputs to the output.  In the example where the input
is an image, and the network provides scores for the presence of an
element (dog, cat,...), the inputs (pixels) that contribute most to
the output are expected to be the ones in the area of the image where
that element appears.  This provides a way to locate that element
within the image, e.g. where exactly the dog or the cat appears in the
image.

IG was designed requiring it to satisfy a number of desirable
properties, particularly \emph{sensitivity} and \emph{implementation
  invariance}.  To define sensitivity we will need to establish a
baseline input where the element or feature is absent (in image
recognition a black image $\mathbf{x}^{\text{base}} = (0,\dots,0)$ may
serve the purpose).  Then, the definitions of sensitivity and
implementation invariance are as follows.

\begin{itemize}

\item An attribution method satisfies \emph{Sensitivity} if for every
  input and baseline that differ in one feature but have different
  predictions then the differing feature should be given a non-zero
  attribution.

\item An attribution method satisfies \emph{Implementation
    Invariance} if the attributions are always identical for two
  functionally equivalent networks (two networks are functionally
  equivalent if their outputs are equal for all inputs).
  
\end{itemize}

The solution proposed in \cite{sundararajan2017ig} is as follows.
Given a baseline $\mathbf{x}^{\text{base}}$ and an input
$\mathbf{x}^{\text{input}}$, then the attribution assigned to the
$i$-th coordinate of the input is:

\begin{equation}\label{e:igi}
\text{IG}_i = (x_i^{\text{input}} - x_i^{\text{base}}) \int_0^1
\frac{\partial F(\gamma(t))}{\partial x_i} \, dt \,,
\end{equation}
where $\gamma(t) = \mathbf{x}^{\text{base}} + t \, (\mathbf{x}^{\text{input}} - \mathbf{x}^{\text{base}})$.

In the next section we will justify this formula, and discuss two
additional properties of~IG:

\begin{itemize}

\item \emph{Completeness.} The attributions add up to the difference
  between the output of $F$ at the input $\mathbf{x}^{\text{input}}$
  and the baseline $\mathbf{x}^{\text{base}}$.

\item \emph{Symmetry Preserving.} Two input variables are symmetric
  w.r.t. a function $F$ if swapping them does not change the
  function. An attribution method is symmetry preserving, if for all
  inputs that have identical values for symmetric variables and
  baselines that have identical values for symmetric variables, the
  symmetric variables receive identical attributions.
  
\end{itemize}

\subsection{IG Completeness and Symmetry Preserving}

\subsubsection{Completeness}

The solution proposed by the authors can be understood as an
application of the Gradient Theorem for line integrals.  Under
appropriate assumptions on function $F$ we have:
\begin{equation}\label{e:gradthm}
    F(\gamma(1)) - F(\gamma(0)) = \int_{\gamma} \nabla F(\mathbf{x}) \cdot \, d\mathbf{x} =
    \int_{\gamma} \sum_{i=1}^n \frac{\partial F(\mathbf{x})}{\partial x_i} \, dx_i
    = \sum_{i=1}^n \int_{\gamma} \frac{\partial F(\mathbf{x})}{\partial x_i} \, dx_i
    \,,
\end{equation}
where $\gamma : [0,1] \to \mathbb{R}^n$ is a smooth (continuously
differentiable) path, and $\nabla$ is the nabla operator, i.e.,
$\nabla F= \left(\frac{\partial F}{\partial x_1},\dots,\frac{\partial
    F}{\partial x_n}\right)$.
The attribution $\text{IG}_i$ to the $i$-th variable $x_i$ of the
score increase with respect to baseline $F(\mathbf{x}^{\text{input}}) - F(\mathbf{x}^{\text{base}})$ is
the $i$-th term of the final sum, i.e.,
$\text{IG}_i = \int_{\gamma} \frac{\partial F(\mathbf{x})}{\partial
  x_i} \, dx_i$.

Note that the result is highly dependent on the path $\gamma$ chosen.
The authors of IG claim that the best (in fact \emph{canonical})
choice of path is a straight line from baseline to input, i.e.,
$\gamma(t) = \mathbf{x}^{\text{base}} + t \,
(\mathbf{x}^{\text{input}} - \mathbf{x}^{\text{base}})$.
With this choice we get $dx_i = (x_i^{\text{input}} - x_i^{\text{base}}) \, dt$, and
equation (\ref{e:igi}) follows.

If (\ref{e:gradthm}) holds, then it is easy to see that IG also
satisfies the completeness property:
\[
  F(\mathbf{x}^{\text{input}}) - F(\mathbf{x}^{\text{base}}) = \sum_{i=1}^n \text{IG}_i
  \,.
\]
This is Proposition~1 in their paper.  However this result depends on
the Gradient Theorem for line integrals, which requires the function
$F$ to be continuously differentiable everywhere.  This cannot be the
case for deep networks, which involve functions such as ReLU and max
pooling that are not everywhere differentiable.  As a fix the authors
restate the Gradient Theorem assuming only that $F$ is continuous
everywhere and differentiable almost everywhere.  However this cannot
work, the particular case of the Gradient Theorem in 1-dimension is
the (second) Fundamental Theorem of Calculus, which does not hold in
general under those premises---Cantor's staircase function provides a
well known counterexample.  There are also issues concerning allowable
paths, e.g. for the function $f(x,y) = \text{max}(x,y)$ the partial
derivatives ${\partial f}/{\partial x}$ and
${\partial f}/{\partial y}$ don't even exist at the points of the line
$x=y$.

Our Proposition~\ref{Lp1} in the next section states a version of the
Gradient Theorem that can be applied to deep networks.

\subsubsection{Symmetry Preserving}

Theorem~1 of \cite{sundararajan2017ig} states that IG, with paths that
are a straight line from baseline to input, is the unique path method
that is symmetry-preserving.  However this theorem is not stated in
the paper with full mathematical rigor, and the proof provided
contains some inconsistencies.\footnote{Look for instance at the value
  of function $f$ in the region where $x_i \leq a$ and $x_j \geq b$.}
Here we will present a completely rigorous formulation of a theorem
that we believe captures the original authors' intention, and provide
its full proof.

\section{Main Results}

Here we state an appropriate generalization of the Gradient Theorem
that can be applied to deep networks, and study the
symmetry-preserving property of IG with straight-line paths.

First, we need to extend the class of functions to which we want to
apply the theorem so that it includes functions implemented by common
deep networks.  We will do so by introducing Lipschitz continuous
functions.

\begin{definition}
  A function $F : S \subseteq \mathbb{R}^n \to \mathbb{R}^m$ is
  \emph{Lipschitz continuous} if there is a constant $K\geq 0$ such
  that
  $||f(\mathbf{x}) - f(\mathbf{y})|| \leq K ||\mathbf{x} -
  \mathbf{y}||$
  for every $\mathbf{x}, \mathbf{y} \in S$, where $||\cdot||$
  represents Euclidean distance.
\end{definition}

For univariate functions, continuity and almost everywhere
differentiability is a necessary condition to make sense of the
integral of a derivative.  The following result ensures that such
condition is satisfied for multivariate Lipschitz continuous
functions.

\begin{thmrade} 
  If $U$ is an open subset of $\mathbb{R}^n$ and
  $F : U \to \mathbb{R}^m$ is Lipschitz continuous, then f is
  differentiable almost everywhere in $U$.
\end{thmrade}

\begin{proof}
  See e.g. \cite{federer1969geo}, Theorem~3.1.6., or
  \cite{heinonen2004lip} Theorem~3.1.
\end{proof}

Rademacher's theorem ensures that the function in our
proposition~\ref{Lp1}, that we state next, is differentiable almost
everywhere.  However that does not mean that such function is
differentiable almost everywhere on a given path---e.g. the function
$\text{max}(x,y)$ is everywhere continuous, and almost everywhere
differentiable on $\mathbb{R}^2$, but is not differentiable at any
point of the line $x=y$.  So differentiability on the path needs to be
included as an additional premise.

\begin{prop}[Gradient Theorem for Lipschitz Continuous Functions]\label{Lp1} 
  Let $U$ be an open subset of $\mathbb{R}^n$.  If
  $F : U \to \mathbb{R}$ is Lipschitz continuous, and
  $\gamma : [0,1] \to U$ is a smooth path such that $F$ is
  differentiable at $\gamma(t)$ for almost every $t\in [0,1]$, then
  \[
    \int_{\gamma} \nabla F(\mathbf{x}) \cdot d\mathbf{x} = F(\gamma(1)) - F(\gamma(0))
    \,.
  \]
\end{prop}
\begin{proof}
  The path $\gamma$ is continuously differentiable on a compact set
  (the interval $[0,1]$), hence it is Lipschitz continuous (because
  its derivative is continuous and so bounded on $[0,1]$).  The
  composition of two Lipschitz continuous functions is Lipschitz
  continuous, hence $t \mapsto F(\gamma(t))$ is Lipschitz continuous,
  which implies absolutely continuous.  By the Fundamental Theorem of
  Calculus for absolutely continuous functions\footnote{See
    \cite{kolmogorov1970realanalysis}~sec.\,33.2, theorem~6.} we have
  \[
    F(\gamma(1)) - F(\gamma(0)) = \int_0^1 \frac{d}{dt} F(\gamma(t)) \, dt
    \,.
  \]
  By the multivariate chain rule we have
  \[
    \frac{d}{dt} F(\gamma(t)) = \nabla F(\gamma(t)) \cdot \gamma'(t)
  \]
  wherever $F$ is differentiable (for almost every $t\in [0,1]$ by
  hypothesis).  Hence
  \[
    \int_0^1 \frac{d}{dt} F(\gamma(t)) \, dt = \int_0^1 \nabla F(\gamma(t)) \cdot \gamma'(t) \, dt
    = \int_{\gamma} \nabla F(\mathbf{x}) \cdot d\mathbf{x}
    \,,
  \]
  and the result follows.

\end{proof}

Next, we will look at results intended to capture the
symmetry-preserving properties of~IG.

\begin{definition}
  A multivariate function $F$ with $n$ variables is \emph{symmetric}
  in variables $x_i$ and $x_j$, $i\neq j$, if
  $F(x_1,\dots,x_i,\dots,x_j,\dots,x_n) =
  F(x_1,\dots,x_j,\dots,x_i,\dots,x_n)$.
\end{definition}

\begin{definition}\label{d:sympres}
  The smooth path $\gamma : [0,1] \to \mathbf{R}^n$ is said to be
  \emph{IG-symmetry preserving} for variables $x_i$ and $x_j$ if
  $\gamma_i(0) = \gamma_j(0)$ and $\gamma_i(1) = \gamma_j(1)$ implies
  $\int_{\gamma} \frac{\partial F(\mathbf{x})}{\partial x_i} \, dx_i =
  \int_{\gamma} \frac{\partial F(\mathbf{x})}{\partial x_j} \, dx_j$
  for every function $F$ verifying the hypotheses of proposition~\ref{Lp1}
  that is symmetric in its variables $x_i$ and $x_j$.
\end{definition}

\begin{definition}
  A path $\gamma : [0,1] \to \mathbb{R}^n$ is \emph{monotonic} if for
  each $i=1,\dots,n$ we have
  $t_1 \leq t_2 \Rightarrow \gamma_i(t_1) \leq \gamma_i(t_2)$ or
  $t_1 \leq t_2 \Rightarrow \gamma_i(t_1) \geq \gamma_i(t_2)$, where
  $\gamma_i(t)$ represents the $i$-th coordinate of
  $\gamma(t)$.\footnote{In other words, each $\gamma_i$ is either
    increasing, or decreasing.  Note that $\gamma$ could be increasing
    in some coordinates and decreasing in other}  The path $\gamma$ is
  strictly monotonic if the inequalities hold replacing them with
  strict inequalities, i.e.,
  $t_1 < t_2 \Rightarrow \gamma_i(t_1) < \gamma_i(t_2)$ or
  $t_1 < t_2 \Rightarrow \gamma_i(t_1) > \gamma_i(t_2)$.
\end{definition}

Next theorem is intended to capture the symmetry-preserving properties of IG.
The proof follows closely the one given in \cite{sundararajan2017ig}.

\begin{thm}\label{Lt1}

  Given $i,j\in \{1,\dots,n\}$, $i\neq j$, real numbers $a<b$, and a
  strictly monotonic smooth path $\gamma : [0,1] \to (a,b)^n$
  such that $\gamma_i(0)=\gamma_j(0)$ and $\gamma_i(1)=\gamma_j(1)$,
  then the following statements are equivalent:
  
  \begin{itemize}
    
  \item[(1)] For every $t\in [0,1]$, $\gamma_i(t) = \gamma_j(t)$.

  \item[(2)] For every function $F : [a,b]^n \to \mathbb{R}$ symmetric
    in $x_i$ and $x_j$ and verifying the premises of
    proposition~\ref{Lp1} with $U=(a,b)^n$ we have
    $\int_{\gamma} \frac{\partial F(\mathbf{x})}{\partial x_i} \, dx_i
    = \int_{\gamma} \frac{\partial F(\mathbf{x})}{\partial x_j} \,
    dx_j$
    (i.e., $\gamma$ is IG-symmetry preserving for variables $x_i$ and
    $x_j$).
  
  \end{itemize}
  
\end{thm}
\begin{proof}
  
  Proof of (1) $\Rightarrow$ (2).  Since $\gamma_i(t) = \gamma_j(t)$
  for every $t\in [0,1]$, and $F$ is symmetric with respect to variables
  $x_i$ and $x_j$, we have
  $\frac{\partial F(\gamma(t))}{\partial x_i} =
  \frac{\partial F(\gamma(t))}{\partial x_j}$
  for almost every $t\in [0,1]$.  Hence they have the same integral.

  Proof of (2) $\Rightarrow$ (1).  Without loss of generality we will
  assume that $\gamma_i$ and $\gamma_j$ are increasing, so that
  $\gamma_i(0)=\gamma_j(0) < \gamma_i(1)=\gamma_j(1)$.  Next, assume
  that (1) is not true. Then, for some $t_0 \in (0,1)$ we have
  $\gamma_i(t_0) \neq \gamma_j(t_0)$.  Assume wlog
  $\gamma_i(t_0) < \gamma_j(t_0)$.  Let $(u,v)$ be the maximum
  interval containing $t_0$ such that $\gamma_i(t) < \gamma_j(t)$ for
  every $t \in (u,v)$.  Since $(u,v)$ is maximum, and $\gamma_i$,
  $\gamma_j$ are increasing, then
  $\gamma_i(t),\gamma_j(t) < \gamma_i(u) = \gamma_j(u)$ for $t<u$, and
  $\gamma_i(t),\gamma_j(t) > \gamma_i(v) = \gamma_j(v)$ for $t>v$.

  Define $g : [a,b] \to \mathbb{R}$ as follows:
  \[
    g(x) =
    \begin{cases}
      0 & \text{ if } a \leq x < u \\
      x-u & \text{ if } u \leq x \leq v \\
      v & \text{ if } v < x \leq b
    \end{cases}
  \]
  and $F(\mathbf{x}) = g(x_i) g(x_j)$.  Then $F$ is symmetric in $x_i$
  and $x_j$, and verifies the premises of proposition~\ref{Lp1}.  For
  $t\notin [a,b]$ we have that $F(\gamma(t))$ is constant, hence
  $\frac{\partial F(\gamma(t))}{\partial x_i} = \frac{\partial
    F(\gamma(t))}{\partial x_j} = 0$. For $t \in [u,v]$ we have
  \[
    \begin{aligned}
    \frac{\partial F(\gamma(t))}{\partial x_i} &=
    \frac{\partial}{\partial x_i}((x_i-u)(x_j-u)) =  (x_j-u) = \gamma_j(t) - u \,, \\
    \frac{\partial F(\gamma(t))}{\partial x_j} &=
    \frac{\partial}{\partial x_j}((x_i-u)(x_j-u)) =  (x_i-u) = \gamma_i(t) - u \,.
    \end{aligned}
  \]
  By hypothesis  $\gamma_i(t) < \gamma_j(t)$, hence
  $\int_{\gamma} \frac{\partial F(\mathbf{x})}{\partial x_i} \, dx_i >
  \int_{\gamma} \frac{\partial F(\mathbf{x})}{\partial x_j} \, dx_j$,
  which is a contradiction.

  This completes the proof.
\end{proof}

\begin{cor}
  Let $\mathbf{p} = (p_1,\dots,p_n)$ and
  $\mathbf{q} = (q_1,\dots,q_n)$ be two points in the open set
  $U \subseteq \mathbb{R}^n$, such that $p_i=p_j$ and $q_i=q_j$.  Then,
  the (straight line) path
  $\gamma(t) = \mathbf{p} + t \, (\mathbf{q} - \mathbf{p})$ is
  IG-symmetry preserving for variables $x_i$, $x_j$ for every function
  $F$ that is symmetric in $x_i$ and $x_j$ and verifies the hypotheses
  of proposition~\ref{Lp1}.  Furthermore, if $p_{j_1}=\dots=p_{j_r}$
  and $q_{j_1}=\dots=q_{j_r}$, and $\gamma$ is IG-symmetric preserving
  for (each pair of) variables $x_{i_1},\dots,x_{i_r}$ in the sense of
  definition~\ref{d:sympres}, then
  $t \mapsto (\gamma_{i_1}(t),\dots,\gamma_{i_r}(t))$ is a straight
  line in $\mathbb{R}^r$.
\end{cor}

Next, we will look at a simple example illustrating the result stated in 
theorem~\ref{Lt1}.

\begin{example}
Figure~(\ref{f:example}) illustrates the failure to preserve symmetry
when the path is not a straight line between baseline and input.
The function used is $F(x_1,x_2) = x_1 x_2$, which is symmetric,
and the path is a curve joining $(0,0)$ and $(1,1)$.
The attributions from IG are
\[
\begin{aligned}
\text{IG}_1 &= \int_{\gamma} \frac{\partial F(x_1,x_2)}{\partial x_1} \,dx_1 
=  \int_o^1 x_2 \,dx_1 \,, \\
\text{IG}_2 &= \int_{\gamma} \frac{\partial F(x_1,x_2)}{\partial x_2} \,dx_2
=  \int_0^1 x_1 \,dx_2
\,.
\end{aligned}
\]\begin{figure}[htb]\label{example}
\centering
\includegraphics[width=2in]{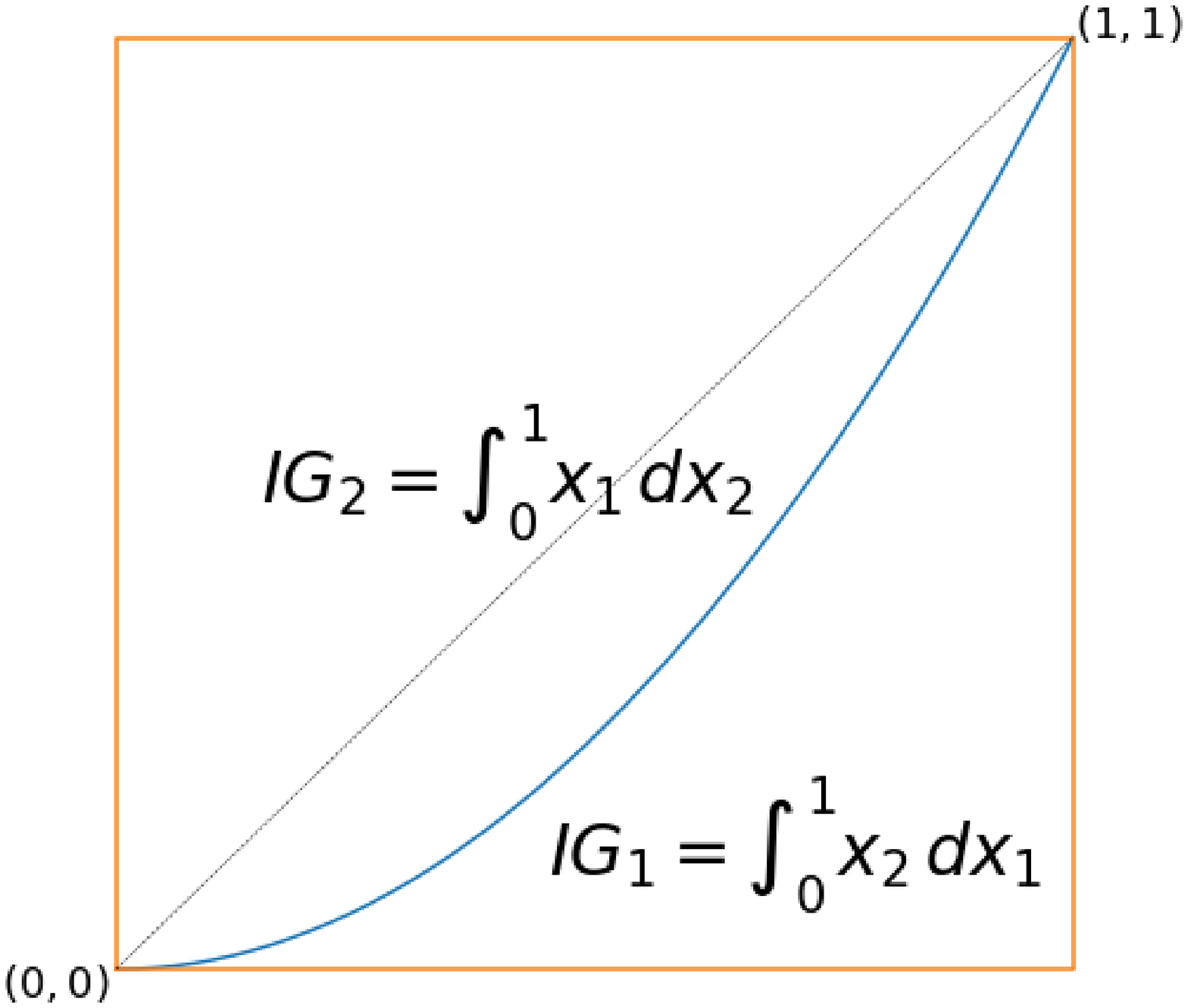}
\caption{Example.}\label{f:example}
\end{figure}
We have $\text{IG}_1 =$ area under the path $\gamma$, and
$\text{IG}_2 =$ area above the path. Their sum is
$\text{IG}_1 + \text{IG}_2 = 1$, the whole area of the square with
vertices $(0,0)$ and $(1,1)$, which equals $F(1,1)=1$, and the
Gradient Theorem holds.  However $\text{IG}_1 < \text{IG}_2$.  If we
used instead the straight line path joining $(0,0)$ and $(1,1)$, then
the IG attributions would be equal.
\end{example}

Note that the uniqueness of IG (using straight line paths) is not
fully captured by our results, and in general does not hold (under the
definition of ``symmetry-preserving'' as worded in the IG paper, which
we tried to formally capture in definition~\ref{d:sympres}).  All we
can tell is that if $p_i=p_j$ and $q_i=q_j$ then
$\gamma_i(t) = \gamma_j(t)$ for every $t\in [0,1]$, but if the
premises fail (i.e., $p_i\neq p_j$ or $q_i\neq q_j$) then nothing
forces the path to be a straight line.  To illustrate this point we
give next a counterexample showing a way to select paths that are not
straight lines in general, and still verify the definition of
IG-symmetry preserving.

\begin{counterex}
  Consider the following path between points $\mathbf{p}=(p_1,p_2)$
  and $\mathbf{q}=(q_1,q_2)$ in $\mathbb{R}^2$:
\[
\begin{aligned}
  \gamma_1(t) &= p_1 + t \, (q_1 - p_1) + t(t-1)((p_1-p_2)^2 + (q_1-q_2)^2) \,\, \text{sgn}(q_1 - p_1) \,, \\
  \gamma_2(t) &= p_2 + t \, (q_2 - p_2) + t(t-1)((p_1-p_2)^2 + (q_1-q_2)^2) \,\, \text{sgn}(q_2 - p_2)
  \,.
\end{aligned}
\]
We multiply the last term of each expression by the sign
function\footnote{The sign function is defined
  $\text{sgn}(x)=\frac{x}{|x|}$ if $x\neq 0$, and $\text{sgn}(0)=0$.}
to make sure that the paths are monotonic (this shows that requiring
the paths to be monotonic does not affect the result.) Also note that
the assignment of path $(\mathbf{p},\mathbf{q}) \mapsto \gamma$, where
$\gamma$ is defined as above, is symmetric in the sense that swapping
the indexes $1$ and $2$ in the expression produces another expression
that is equivalent to the original, so the assignment of path is
symmetric with respect to the coordinates $x_1$ and $x_2$.  Also, we have
that if $p_1=p_2$ and $q_1=q_2$ then
$\gamma(t) = \mathbf{p} + t \, (\mathbf{q} - \mathbf{p})$, which,
according to theorem~\ref{Lt1}, is IG-symmetry preserving for
variables $x_1$, $x_2$ for every function $F$ that is symmetric in
$x_1$ and $x_2$ and verifies the hypotheses of proposition~\ref{Lp1}.
However, if $p_1 \neq p_2$ or $q_1 \neq q_2$, the quadratic terms in
$t$ make $\gamma$ a curve that is not a straight line in general, and
hence it differs from the path used in the IG attribution method.  The
symmetry preserving property is not violated because in the cases
where the path is not a straight line the premises of the definition
don't apply, so theorem~\ref{Lt1} still holds.
\end{counterex}

Admittedly this counterexample is an artificial modification of a
straight-line. IG (with straight line paths) is a simple path-based
symmetry-preserving attribution method, and we see no reason to
replace it with a different method using non straight-line paths
without justification.

\section{Conclusions}

We have rigorously stated and proved that Integrated Gradients has
completeness and symmetry preserving properties.  The premises used to
prove the result makes it suitable for functions implemented by common
deep networks.

On the other hand we have shown that IG with straight line paths is
\emph{not} the unique path method that is symmetry-preserving, in fact
there are path methods that verify the definition of
symmetry-preserving but don't necessarily use straight line paths for
all combinations of baseline and input.  Note that this should not be
taken as an argument against using straight line paths in IG, in fact
straight lines are still the simplest paths that provide the desired
results.


\begin{thebibliography}{0}

\bibitem{federer1969geo} Federer, Herbert (1969). Geometric measure
  theory, Die Grundlehren der mathematischen Wissenschaften, 153,
  Berlin–Heidelberg–New York: Springer-Verlag, pp. xiv+676, ISBN
  978-3-540-60656-7, MR 0257325, Zbl 0176.00801. 

\bibitem{heinonen2004lip} Heinonen, Juha (2004). Lectures on Lipschitz Analysis. \hfill \break
URL: {\tt http://www.math.jyu.fi/research/reports/rep100.pdf}

\bibitem{kolmogorov1970realanalysis} A.\,N.\,Kolmogorov, S.\,V.\,Fomin
  (1970). Introductory Real Analysis. Dover Books on Mathematics.
  

\bibitem{simonyan2015vgg} Karen Simonyan, Andrew Zisserman
  (2015). Very Deep Convolutional Networks for Large-Scale Image
  Recognition.  arXiv preprint {\tt arXiv:1409.1556 [cs.CV]}


\bibitem{sundararajan2017ig} Mukund Sundararajan, Ankur Taly, and Qiqi
  Yan (2017).  Axiomatic attribution for deep networks.  arXiv
  preprint {\tt arXiv:1703.01365 [cs.LG]}

\end{thebibliography}
\end{document}